\documentclass{article}
\usepackage{amsmath,amssymb,amsthm}
\usepackage[dvipdfmx]{graphicx} 
\usepackage{float}
\theoremstyle{definition}

\theoremstyle{plain}
\newtheorem{proposition}{Proposition}
\newtheorem{lemma}[proposition]{Lemma}
\newtheorem{theorem}[proposition]{Theorem}

\begin{document}
\title{Equilibrium Points of an AND-OR Tree: under Constraints on Probability
}
\author{Toshio Suzuki
${}^1$\thanks{
Corresponding author. 
This work was partially supported by Japan Society for the Promotion of Science (JSPS) KAKENHI (C) 22540146 and (B) 23340020.}
and  
Yoshinao Niida ${}^2$\thanks{
The current affiliation is Patent Result Co., Ltd., Japan.
}
\\ 
Department of Mathematics and Information Sciences, \\ 
Tokyo Metropolitan University, \\ 
Minami-Ohsawa, Hachioji, Tokyo 192-0397, Japan\\
1: toshio-suzuki@tmu.ac.jp
\quad 
2: yn.sputnik@gmail.com 
}
\date{\today}

\maketitle              

\begin{abstract}
We study a probability distribution $d$ on the truth assignments to a uniform binary AND-OR tree. Liu and Tanaka [2007, Inform. Process. Lett.] showed the following: If $d$ achieves the equilibrium among independent distributions (ID) then $d$ is an independent identical   distribution (IID). We show a stronger form of the above result. Given a real number $r$ such that $0 < r < 1$, we consider a constraint that the probability of the root node having the value 0 is $r$. Our main result is the following: When we restrict ourselves to IDs satisfying this constraint, the above result of Liu and Tanaka still holds. The proof employs clever tricks of induction. In particular, we show two fundamental relationships between expected cost and probability in an IID on an OR-AND tree: (1) The ratio of the cost to the probability (of the root having the value 0) is a decreasing function of the probability $x$ of the leaf. (2) The ratio of derivative of the cost to the derivative of the probability is a decreasing function of $x$, too.

\end{abstract}

\section{Introduction}

The alpha-beta pruning is a well-known algorithm for tree searching. 
Knuth and Moore \cite{KM75} pioneered analysis of the alpha-beta pruning. 
Baudet \cite{Ba78} and Pearl \cite{Pe80} has studied optimality of alpha-beta pruning in the case where values of terminal nodes are independent identically distributed. 
The optimality is established by Pearl \cite{Pe82} and Tarsi \cite{Ta83}. 
For more on important early works, see the references of \cite{Pe82}. 

We are interested in the case of a uniform binary tree such that each leaf is bi-valued. 
In this special case, a mini-max tree makes a binary AND-OR tree; And, the alpha-beta pruning 
is described in a simple way. Given an AND-node (an OR-node, respectively) $v$, if we know a child of $v$ has the value 0 (1) then we know that $v$ has the same value, without probing the other child. Here, a cut-off (or, a skip) happens at the other child. 

And, we are interested in the case where an associated probability distribution $d$ on the truth assignments to the leaves is an independent distribution (ID) but $d$ is not necessarily an independent identical distribution (IID).  
Here, an IID denotes an ID such that all the leaves have the same probability of having the value 0.  

Yao's principle \cite{Ya77}, a variation of von-Neumann's minimax theorem, is useful for analyzing equilibriums of AND-OR trees. 
Saks and Wigderson \cite{SW86} establish basic results on the equilibriums. 
Liu and Tanaka \cite{LT07} have extended the works of Yao and Saks-Wigderson.  
And, they study the eigen-distribution, the distribution achieving the equilibrium. 
In the course of their study, Liu and Tanaka showed the following.  

\begin{theorem} \label{thm:lt4}
 $($Liu and Tanaka, Theorem 4 of $\cite{LT07})$  
If $d$ is an eigen-distribution with respect to IDs then $d$ is an IID. 
\end{theorem}

They write ``it is not hard" to show the theorem, and omit the proof. 

In this paper, we show a stronger form of the above theorem. 
Throughout the paper, a probability of a given node denotes the probability that the node has the value 0. 

\vspace{\baselineskip}

\noindent
\textbf{Main Theorem} (Theorem~\ref{thm:main})
\textit{
Suppose that $r$ is a real number such that $0 < r < 1$. 
Suppose that we restrict ourselves to distributions such that the probability of the root is $r$. Under this constraint, Theorem~\ref{thm:lt4} still holds. 
}

\vspace{\baselineskip}

Our proof of Theorem~\ref{thm:main} employs clever tricks of induction. 
In particular, we show the following two lemmas by induction. 

\vspace{\baselineskip}

\noindent
\textbf{Lemma}~\ref{lem:1}
\textit{
Suppose that an OR-AND tree is given. 
Given an $x$ $(0 < x < 1)$, we consider an IID such that each leaf has probability $x$. Then, the following quantity is a decreasing function of $x$. Here, the numerator $($the denominator$)$ denotes the expected value of the cost $($the probability of the root, respectively$)$. 
}

\begin{equation*}
\frac{\mathrm{cost} (x)}{\mathrm{prob} (x)}
\end{equation*}

\noindent
\textbf{Lemma}~\ref{lem:2}
\textit{
Under the same assumption of Lemma~\ref{lem:1}, 
the following quantity is a decreasing function of $x$. 
Here, the primes denote differentiation.
}
\begin{equation*}
\frac{\mathrm{cost}^\prime (x)}{\mathrm{prob}^\prime (x)}
\end{equation*}

Lemmas~\ref{lem:1} and \ref{lem:2} describe fundamental relationships between the cost and the probability in an IID on an OR-AND tree. 
In section~\ref{section:lemmas}, we prove these lemmas. 
In section~\ref{section:main}, we prove Theorem~\ref{thm:main} by using Lemmas~\ref{lem:1} and \ref{lem:2}. 
We show Theorem~\ref{thm:lt4} as a corollary to Theorem~\ref{thm:main} in section~\ref{section:original}. 
The main motive for Theorem~\ref{thm:lt4} is the following (Theorem 9 of \cite{LT07})
: The equilibrium among all IDs is strictly smaller than the equilibrium among all ditributions.  
In section~\ref{section:comparison}, we observe that the above result still holds under the constraint on the probability of the root. 
In section~\ref{section:conclusiveremarks}, we discuss whether Theorem~\ref{thm:lt4} is really ``not hard". We observe that a brutal induction does not work for the proof of Theorem~\ref{thm:lt4}. 
 
\section{Notation and Conventions}

A tree is said to be an \emph{AND-OR tree}\/ (an \emph{OR-AND tree}, respectively) if the root is an AND-node (an OR-node, respectively), and OR layers and AND layers alternate. A \emph{leaf}\/ means a terminal node. 
Throughout the paper, unless specified, a tree is assumed to be a uniform binary tree. In other words, every internal node has just 2 child nodes, and all the leaves have the same distance from the root. 

Figure~\ref{fig:1} shows an example of an AND-OR tree of height 2. If we exchange the roles of AND-gates ($\wedge$) and those of OR-gates ($\vee$), the resulting tree is an OR-AND tree of height 2. 

\begin{figure}[H] 
\begin{center}
\includegraphics[width=.2\linewidth]{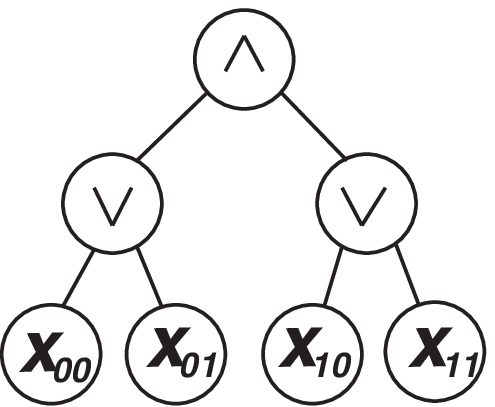}
\caption{AND-OR tree of height 2 \label{fig:1}}
\end{center}
\end{figure}

A \emph{truth assignment}\/ for a tree is a mapping from the set of all leaves to $\{ 0, 1 \}$. Here, 0 stands for FALSE and 1 stands for TRUE. 
An example of a truth assignment for the tree in Figure~\ref{fig:1} is as follows: $f(x_{00})=f(x_{01})=f(x_{01})=1$ and $f(x_{11})=0$. 
Throughout the paper, a distribution denotes a probability distribution on the truth assignments for a given tree. 

We consider algorithms finding the value of the root. An algorithm makes some queries to leaves. We concentrate on deterministic algorithms: An algorithm does not perform coin tossing. An algorithm may be directional or un-directional. An algorithm is said to be \emph{directional}\/ if for some linear arrangement of the leaves it never selects for examination a node situated to the left of a previously examined node \cite[p.121]{Pe80}. Otherwise, the algorithm is \emph{un-directional}. 

An \emph{alpha-beta pruning algorithm}\/ for an AND-OR tree (or an OR-AND tree) is defined  in Introduction.  
Throughout the paper, a deterministic algorithm is assumed to be a deterministic alpha-beta pruning algorithm on a given tree. 

Given an algorithm and a truth assignment, their \emph{cost}\/ denotes the number of leaves probed during the computation. 
Given an algorithm $A_D$ and a distribution $d$, their \emph{cost}\/ denotes the expected value of the cost. We denote the cost by $C(A_D, d)$. 

Given a class $B$ of distributions, a distribution $d_0 \in B$ is an \emph{eigen-distribution}\/ with respect to $B$ if $d$ is a maximizer of the minimum cost \cite{LT07}. To be more precise, if the following holds. 
\begin{equation}
\min_{A_D} C(A_D, d_0) = \max_d \min_{A_D} C(A_D, d) 
\end{equation}
Here, $A_D$ runs over all deterministic algorithms. And, $d$ runs over all elements of $B$. 

Throughout the paper, unless specified, in expressions such as $\min_{A_D}$ and $\max_{A_D} $, $A_D$ runs over all deterministic algorithm on a given tree. 

Let $h$ be a positive integer, and $T$ a uniform binary AND-OR tree of height $h$.  
Given a real number $x$ ($0 \leq x \leq 1$), we consider an ID on $T$ such that each leaf has probability (having the value 0) $x$. 
Then,  $p_{\wedge, h} (x)$ ($c_{\wedge, h} (x)$, respectively) denotes the probability (the expected cost, respectively) of the root. 

Given a uniform binary OR-AND tree of height $h$, we define $p_{\vee, h} (x)$ and $c_{\vee, h} (x)$ in the same way.

\section{Relationships between Costs and Probabilities} \label{section:lemmas}

\begin{lemma} \label{lem:1} 
Suppose that $h$ is a positive integer. 
Then, the following quantity is a decreasing function of $x$ $(0 < x < 1)$. 
\begin{equation}
\frac{ c_{\vee, h} (x)}{ p_{\vee, h}  (x)}
\end{equation}
\end{lemma}

\begin{proof} 
The cases of height = 1, 2 are easy. 

As a preliminary observation for an induction step, consider an OR-AND tree of height 2. And, consider an IID in which each leaf has probability $x$, where $0 < x < 1$. 
Then, the cost and probability are given as follows. 

\begin{align}
c_{\vee, 2} (x) &= (2 - x) ( - x^2 + 2 x + 1) \label{eq:orandh2cost}
\\
p_{\vee, 2} (x) &= x^2 ( x - 2 )^2 \label{eq:orandh2prob}
\end{align}

Now, we are going into an induction step. Let $h$ be a positive integer. 
We consider an OR-AND tree of height $h+2$. By \eqref{eq:orandh2cost} and \eqref{eq:orandh2prob}, we have the followings. 

\begin{align}
c_{\vee, h+2} (x) &= (2 - p_{\vee, h} (x)) ( - p_{\vee, h} (x)^2 + 2 p_{\vee, h} (x) + 1) \times c_{\vee, h} (x) \label{eq:orandhplus2cost}
\\
p_{\vee, h+2} (x) &= p_{\vee, h} (x)^2 (p_{\vee, h} (x) - 2)^2
\label{eq:orandhplus2prob}
\end{align}

Therefore, the following holds. 

\begin{equation} \label{eq:lastformulalem3}
\frac{c_{\vee, h+2} (x) }{p_{\vee, h+2} (x) } = 
\Bigl{(}
1 + 
\frac{1}{ p_{\vee, h} (x) (2 - p_{\vee, h} (x)) }
\Bigr{)}
\times
\frac{c_{\vee, h} (x) }{p_{\vee, h} (x) } 
\end{equation}

Both of the two factors in the right-hand side are positive. 
When $x$ varies from 0 to 1, $p_{\vee, h} (x)$ is an increasing function that varies from 0 to 1. 
Thus, the first factor is decreasing. 
And, the second factor is decreasing, by the induction hypothesis. 
Hence, the left-hand side of \eqref{eq:lastformulalem3} is decreasing. 
\end{proof}

Consider the dual of the tree. Then the numerators of both sides of \eqref{eq:prop1a} are the same. And, the same thing holds for the denominators. 
Therefore, the followings hold for $x$ such that $0 < x < 1$, where primes ($\prime$) denote differentiation. And, $c_{\vee, h}^\prime (1 - x)$ denotes  $(d c_{\vee, h} (t) /dt)|_{t=1-x}$.  

\begin{align}
\frac{c_{\wedge, h} (x) }{ 1 - p_{\wedge, h} (x) } 
&= 
\frac{c_{\vee, h} (1 - x) }{p_{\vee, h} (1 - x) } \label{eq:prop1a}
\\
\frac{c_{\wedge, h}^\prime (x) }{ ( 1 - p_{\wedge, h} (x) )^\prime } 
&= 
\frac{c_{\vee, h}^\prime (1 - x) }{p_{\vee, h}^\prime (1 - x) } \label{eq:prop1b}
\end{align}

\begin{lemma} \label{lem:2}
Suppose that $h$ is a positive integer. 
Then, the following quantity is a decreasing function of $x$ ($0 < x < 1$). 
\begin{equation}
\frac{ c_{\vee, h}^\prime (x)}{ p_{\vee, h}^\prime   (x)}
\end{equation}
\end{lemma}

\begin{proof} 
The case of height = 1 is easy. 
As an induction step, we consider an OR-AND tree of height $h+1$. 
The goal is to show that the following quantity is decreasing. 
\begin{equation} \label{eq:goaloflem2a}
\frac{c_{\vee, h+1}^\prime (x) }{ p_{\vee, h+1}^\prime (x) } 
\end{equation}

Let $z := p_{\wedge, h} (x)$. Thus, $z$ is the probability of the node just under the root. 
Therefore, we have the following. 

\begin{equation} \label{eq:goaloflem2b}
\frac{c_{\vee, h+1}^\prime (x) }{ p_{\vee, h+1}^\prime (x) } 
=
\frac{ \dfrac{d c_{\vee, h+1} }{dz} }{ \quad \dfrac{d p_{\vee, h+1}}{dz}  \quad} 
\end{equation}

Here, the numerator (the denominator) of the right-hand side denotes the derivative of 
$c_{\vee, h+1} ( p_{\wedge, h}^{-1} (z) )$ 
(the derivative of $p_{\vee, h+1} ( p_{\wedge, h}^{-1} (z) )$, respectively). 

Now, define a function $c(z)$ ($0 < z < 1$) as follows. 

\begin{equation}
c(z) := c_{\wedge, h} ( p_{\wedge, h}^{-1} (z) ) 
\end{equation}

Then, the followings hold. 

\begin{align}
c_{\vee, h+1} ( p_{\wedge, h}^{-1} (z) ) &= c(z) (1+z) \label{eq:cz1plusz}
\\
p_{\vee, h+1} ( p_{\wedge, h}^{-1} (z) ) &= z^2 \label{eq:zsquare}
\end{align}

Now, we look at the following quantity. 

\begin{equation} \label{eq:ddzfraction}
\frac{d}{dz} \Biggl{(} 
\frac
{\dfrac{ d c_{\vee, h+1} }{dz}}
{\quad \dfrac{ d p_{\vee, h+1} }{dz} \quad} 
\Biggr{)}
=
\frac{[ \cdots ]}{ \Bigl{(} \dfrac{ d p_{\vee, h+1} }{dz} \Bigr{)}^2 }
\end{equation}

By \eqref{eq:cz1plusz} and \eqref{eq:zsquare}, 
the numerator $[ \cdots ]$ of the right-hand side equals to  the following. 

\begin{align}
&\{ c(z) (1+z) \}^{\prime \prime} (z^2)^\prime - \{ c(z) (1+z) \}^\prime (z^2)^{\prime \prime}   
\notag \\
=
&2\{
c^{\prime \prime} (z) (1+z) z +  c^\prime (z) (z - 1) - c(z) 
\} \label{eq:lemkey1}
\end{align}

In the remainder of the proof, we are going to show that the right-hand side of \eqref{eq:lemkey1} is negative. 
By the induction hypothesis and \eqref{eq:prop1b}, the following holds. 

\begin{equation} \label{eq:lem3indhyp}
\frac{d}{dx} \Bigl{(} 
\frac{ c_{\wedge, h}^\prime (x) }{ ( 1 - p_{\wedge, h} (x) )^\prime } 
\Bigr{)}
> 0
\end{equation}

Recall that $dz/dx = p_{\wedge, h}^\prime (x) > 0$ for all $x$ such that $0 < x < 1$. 
Therefore, derivative of a given function by $x$ is positive if and only if derivative of it by $z$ is positive.  
Hence, by \eqref{eq:lem3indhyp}, the followings hold. 

\begin{equation} 
\frac{d}{dz} (- c^\prime (z)) 
= 
\frac{d}{dz} \Biggl{(} 
\frac{ c^\prime (z)}{ \quad \dfrac{d(1-z)}{dz} \quad} \Biggr{)} 
= 
\frac{d}{dz} \Bigl{(} 
\frac{ c_{\wedge, h}^\prime (x) }{ (1 - p_{\wedge, h} (x) )^\prime } 
\Bigr{)}
> 0
\end{equation}

Therefore, we get the following. 

\begin{equation} \label{eq:lemkey2}
c^{\prime \prime} (z) < 0
\end{equation}

On the other hand, by Lemma~\ref{lem:1} and \eqref{eq:prop1a}, the following holds. 

\begin{equation} 
\frac{d}{dx} \Bigl{(} 
\frac{ c_{\wedge, h} (x) }{ 1 - p_{\wedge, h} (x) } 
\Bigr{)}
> 0
\end{equation}

Therefore, we get the following. 

\begin{equation} 
\frac{d}{dz} \Bigl{(} 
\frac{ c (z) }{ 1 - z } 
\Bigr{)}
> 0
\end{equation}

Thus, it holds that $c^\prime (z) ( 1 - z ) - c(z) ( 1 - z )^\prime > 0$. 
Therefore, we get the following. 

\begin{equation} \label{eq:lemkey3}
c^\prime (z) ( z - 1 ) - c (z) < 0
\end{equation}

By \eqref{eq:lemkey1}, \eqref{eq:lemkey2} and \eqref{eq:lemkey3}, it holds that \eqref{eq:ddzfraction} is negative. 
Hence, the right-hand sid of \eqref{eq:goaloflem2b} is a decreasing function of $z$ ($0 < z < 1$). 
Therefore, \eqref{eq:goaloflem2a} is a decreasing function of $x$ ($0 < x < 1$).
\end{proof}

As examples, we show graphs of  
$c_{\vee, 4}(x) / p_{\vee, 4}(x)$ and  $c_{\vee, 4}^\prime (x) / p_{\vee, 4}^\prime (x)$ ($0.1 < x < 0.9$). 

\begin{figure}[H]
\begin{center}
\begin{minipage}[c]{.48\linewidth}
\includegraphics[width=\linewidth]{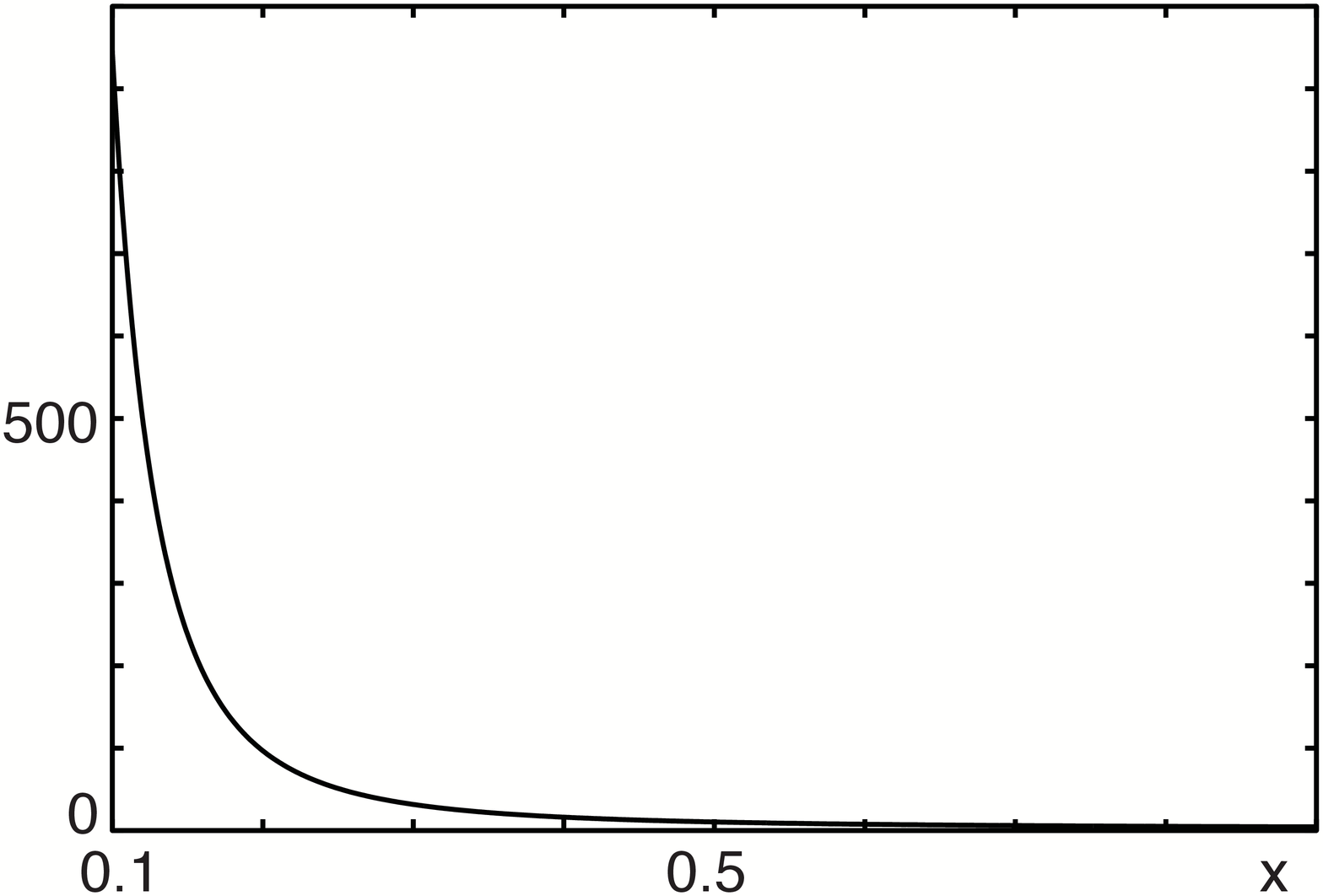}
\caption{$c_{\vee, 4}(x) / p_{\vee, 4}(x)$}
\end{minipage}
\begin{minipage}[c]{.48\linewidth}
\includegraphics[width=\linewidth]{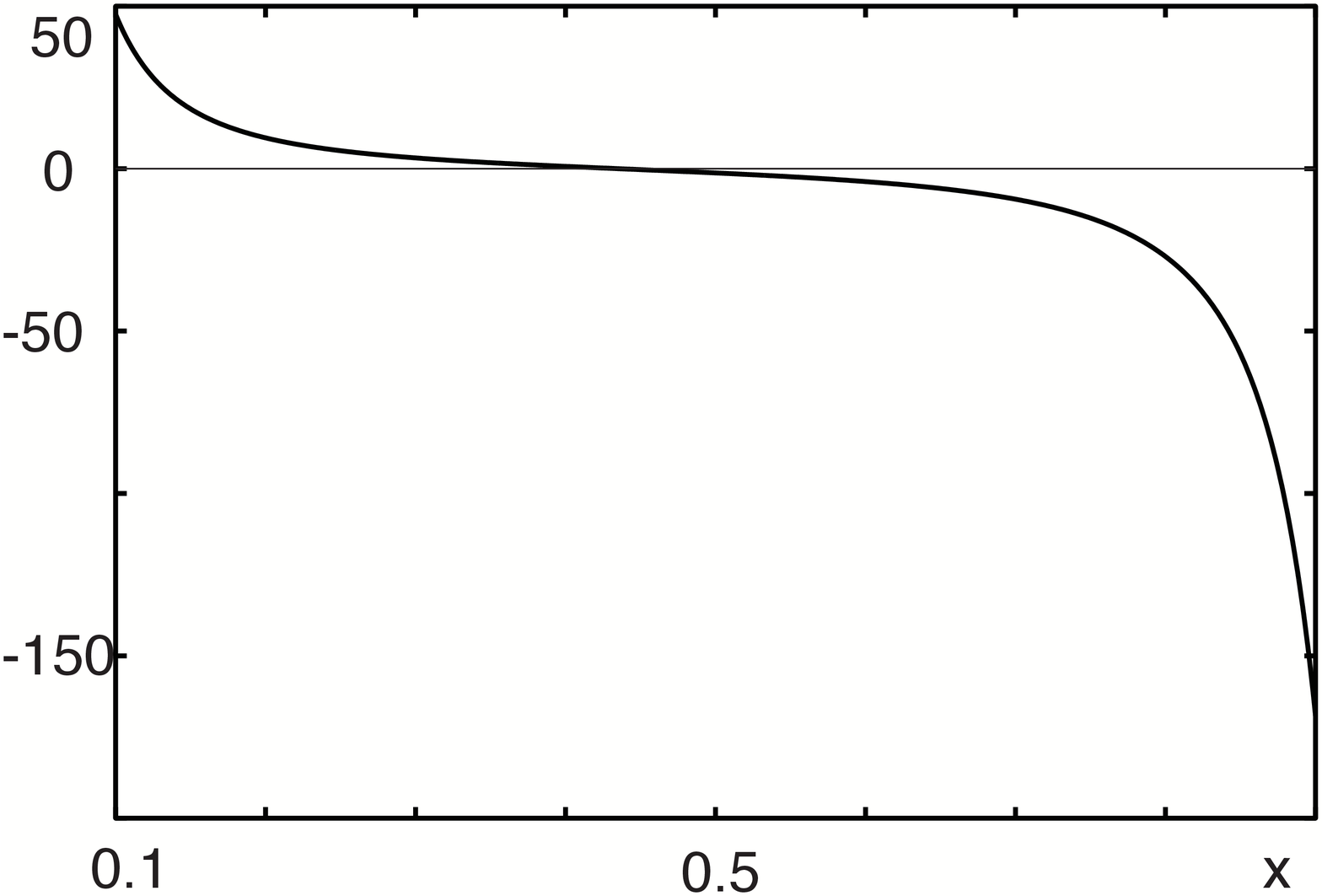}
\caption{$c_{\vee, 4}^\prime (x) / p_{\vee, 4}^\prime (x)$}
\end{minipage}
\end{center}
\end{figure}

\section{Main Theorem} \label{section:main}

\subsection{Constraint Extremum Problem}

We consider the following constraint extremum problem. 

\noindent
\textbf{Constraint Extremum Problem 1}

\textbf{Type of the Problem:} Maximization. 

\textbf{Parameters (Constants): } Let $h$ be a positive integer. And, let $r$ be a real number ($0 < r < 1$). 

\textbf{The Objective Function} 

\begin{equation} \label{eq:objectivefunctiona}
f(z, w) := c_{\vee, h} (p_{\vee, h}^{-1} (z)) + (1 - z) c_{\vee, h} (p_{\vee, h}^{-1} (w))
\end{equation}

\textbf{Side Conditions} 
\begin{align}
& 0 \leq w \leq z \leq r \label{eq:sidecondition11} 
\\ 
& (1 - z) (1 - w) = (1 - r) \label{eq:sidecondition12}
\end{align}

The side condition \eqref{eq:sidecondition11} is equivalent to the following: 
$1 - r \leq 1 - z \leq 1- w \leq 1$. 

It is easy to see that $f(z,w) < f(w,z)$ holds if and only if $w < z$. 
Hence, when we replace the objective function \eqref{eq:objectivefunctiona} by 
$\min \{ f(z, w), f(w,z) \} $, 
the resulting constraint extremum problem is equivalent to Constraint Extremum Problem 1. 

And, it is easy to see the following. 

\begin{equation} 
\min \{ f(z, w), f(w,z) \} = \min_{A_D} C(A_D, d)
\end{equation}

Here, the meaning of the right-hand side is as follows. 
We consider an AND-OR tree of height $h + 1 \geq 2$. 
And, $d$ is the ID defined as follows. The restrictions of $d$ to the left sub-tree and the right sub-tree are IIDs; The left child of the root has the probability $z$ and the other child of the root has the probability $w$. 

\begin{lemma} \label{lem:cep1}
Constraint Extremum Problem 1 has a unique solution 
$(z, w) = (1 - \sqrt{1-r}, 1 - \sqrt{1-r})$.
\end{lemma}

\begin{proof}
Let $c(u): = c_{\vee, h} (p_{\vee, h}^{-1} (u))$ ($0 \leq u \leq r$). 
We define $f_1 (z)$ ($1 - \sqrt{1-r} \leq z \leq r$) as to be $f(z,\omega (z))$, where $\omega (z) := 1 - (1-r)/(1-z)$. 
Thus, we have the following. 

\begin{equation}
f_1 (z)= c(z) + (1 - z) c ( \omega (z))  
\end{equation}

Here, we have $\omega^\prime (z) = -(1-r)/(1-z)^2$. 
Thus, the followings hold. 

\begin{align}
f_1^\prime (z) 
&= c^\prime (z)  - c ( \omega (z)) + c^\prime ( \omega (z) ) (-1) \frac{1-r}{1-z} \notag \\
&= c^\prime (z)  - c ( \omega (z)) + c^\prime ( \omega (z) ) ( \omega (z) - 1) \label{eq:foneprime}
\end{align}

Our goal is to show that $f_1^\prime (z) $ is negative. 
We consider a variable transformation $t = p_{\vee, h}^{-1} (u)$ ($0 \leq u \leq r$). 
Then $dt/du= 1/(du/dt) = 1/(p_{\vee, h}^\prime (t)) > 0$. 
Hence, by Lemma~\ref{lem:1}, the following holds. 

\begin{equation}
\frac{d}{du} \Bigl{(} \frac{c(u)}{u} \Bigr{)} 
= \frac{d}{dt} \Bigl{(} \frac{c_{\vee, h}(t)}{p_{\vee, h}(t)} \Bigr{)} \frac{dt}{du} < 0
\end{equation}

Therefore, in the derived function of $c(u)/u$, the numerator is negative. 
Thus, we have $c^\prime (u)  u -  c (u) < 0$. 
And, the following holds $(0 \leq u \leq r)$. 

\begin{equation} 
c^\prime (u) - c(u) + c^\prime (u) (u - 1) < 0 
\end{equation}

In particular, the following holds $(1 - \sqrt{1-r} \leq z \leq r)$. 

\begin{equation} \label{eq:cprimewminuscplus}
c^\prime (\omega (z)) - c(\omega (z)) + c^\prime (\omega (z)) (\omega (z) - 1)  < 0 
\end{equation}

On the other hand, since $z \geq \omega (z)$, we have $x := p_{\vee, h}^{-1} (z) \geq y := p_{\vee, h}^{-1} (\omega (z)) $. 
By Lemma~\ref{lem:2}, we have the following. 

\begin{equation}
c^\prime (z)
=
\frac	{ \dfrac{dc(z)}{dx} }{\quad \dfrac{dz}{dx} \quad}
=
\frac{ c_{\vee, h}^\prime (x) }{ p_{\vee, h}^\prime (x) }
\leq 
\frac{ c_{\vee, h}^\prime (y) }{ p_{\vee, h}^\prime (y) }
=
c^\prime ( \omega (z) ) 
\end{equation}

Therefore, by \eqref{eq:cprimewminuscplus}, the following holds. 

\begin{equation} \label{eq:thekeyinequality}
c^\prime (z) - c(\omega (z)) + c^\prime (\omega (z)) (\omega (z) - 1)  < 0 
\end{equation}

By \eqref{eq:foneprime} and \eqref{eq:thekeyinequality}, 
it holds that $f_1^\prime (z) < 0$, in the interval $1 - \sqrt{1-r} < z < r$. 

Hence, the unique solution to Constraint Extremum Problem 1 is 
achieved at $(z, w) = (1 - \sqrt{1-r}, 1 - \sqrt{1-r})$. 
\end{proof}

\subsection{Replacement of Sub-trees}

 \begin{proposition} \label{prop:4main} 
Suppose that $T_0$ is an AND-OR tree or an OR-AND tree. Let $i \in \{ 0,1 \}$. And, suppose that $d_0$ is an ID on $T_0$ such that the probability of the root (having the value 0) is $1 - i$; Thus, the root has the value $i$. 
Let $d_1$ be the IID such that every leaf has the probability $i$. 

$(1)$ The following holds. 
\begin{equation} \label{eq:prop4main}
\min_{A_D} C(A_D, d_0) = \min_{A_D} C(A_D, d_1) 
\end{equation} 

$(2)$ 
Let $h$ be the height of the tree. 
In the case where $h$ is even $($denote it by $2k)$, the value of \eqref{eq:prop4main} is $2^k$. 

$(3)$
In the case where $h$ is odd $($denote it by $2k+1)$, 
\begin{equation}
(\mbox{The value of \eqref{eq:prop4main}})
=
\begin{cases}
2^k & \mbox{If $i = 0$ and the root is AND-gate, }
\\
 & \mbox{or $i = 1$ and the root is OR-gate}
\\
2^{k+1} & \mbox{Otherwise.}
\end{cases}
\end{equation}
 \end{proposition}

The above proposition is easily shown by induction on the height. 

 \begin{theorem} \label{thm:main}
 $($\textbf{Main Theorem}$)$ 
Suppose that $T$ is a uniform binary AND-OR tree. 
Suppose that $r$ is a real number such that $0 < r < 1$. 
Now, consider the following set.
\begin{equation} \label{eq:rIDs}
\{ \delta : \delta \mbox{ is an ID on $T$ such that the probability of the root is } r. \}
\end{equation} 

And, suppose that $d$ is in the set \eqref{eq:rIDs} and that $d$ satisfies the following. 
\begin{equation} \label{eq:eigen4rid}
\min_{A_D} C(A_D, d) = \max_{\delta : \mathrm{ID}, r} \min_{A_D} C(A_D, \delta) 
\end{equation}
Here, $\delta$ runs over all elements of \eqref{eq:rIDs}. 
Then, $d$ is an IID. 
\end{theorem}

\begin{proof} 
We prove the theorem by an induction on the height $h$. 
The case of $h = 1$ is easy. 
As an induction step, we consider the case where $h \geq 2$. 

Let $d_L$ ($d_R$, respectively) denote the restriction of $d$ to the left (right) sub-tree. 
And, let $z$ ($w$, respectively) be the probability of the left (right) child of the root. Without loss of generality,  we may assume $w \leq z$. 
It is easy to see that $0 < z < 1$.  
Therefore, by the induction hypothesis, $d_L$ is an IID. 

In the case of $w > 0$, we have $0 < w < 1$. Thus, in the same way as above, $d_R$ is an IID, too.  

In the case of $w = 0$, we do not know whether $d_R$ is an IID. However, by Proposition~\ref{prop:4main}, we can replace $d_R$ by an IID. And, after the replacement, $d$ still satisfies \eqref{eq:eigen4rid}. 

Thus, regardless of whether $w$ is positive or not, we may assume that both of $d_L$ and $d_R$ are IIDs.  Therefore, $(z, w)$ is a solution to Constraint Extremum Problem 1. 
By Lemma~\ref{lem:cep1}, it holds that $z=w$. 
Hence, $d$ is an IID. 
\end{proof}

In Theorem~\ref{thm:main}, the assumption of $0 < r < 1$ is optimal. 
In the case where $r=0$ or $1$, by Proposition~\ref{prop:4main}, there exists an element 
$d$ of the set \eqref{eq:rIDs} such that \eqref{eq:eigen4rid} holds but $d$ is not an IID.

\section{Proof of the Original Theorem} \label{section:original}

The following theorem is asserted in \cite{LT07} without a proof. 
Now, we show it by using our main theorem. 

\noindent
\textbf{Theorem}~\ref{thm:lt4}
\textit{
 $($Liu and Tanaka, Theorem 4 of $\cite{LT07})$ 
Suppose that $T$ is an AND-OR tree. 
And, suppose that $d$ is an ID and that $d$ satisfies the following. 
\begin{equation} \label{eq:eigen4id}
\min_{A_D} C(A_D, d) = \max_{\delta :  \mathrm{ID}} \min_{A_D} C(A_D, \delta) 
\end{equation}
Here, $\delta$ runs over all IDs. Then, $d$ is an IID. 
}

\begin{proof} 
It is enough to show that the probability of the root in $d$ is neither 0 nor 1. 
Then, the proof of the theorem is reduced to our main theorem. 

\textbf{Case 1:} The height $h$ is even. Let $h = 2k$.  
Let $c(z) := c_{\wedge, 2k} ( p_{\wedge, 2k}^{-1} (z))$. 
Then, we have $c(z) > c(0) = c(1)$ in the interval $0 < z < 1$; 
Baudet observed almost same thing \cite[eq.(3.14)]{Ba78}. 
A direct proof is as follows. 
By Proposition~\ref{prop:4main}, it holds that $c (0) = 2^k = c(1)$. 
On the other hand, in the same way as \eqref{eq:lemkey2}, it holds that $c^{\prime \prime} (z) < 0$ ($0 < z < 1$). Therefore, $c(z) > c(0) = c(1)$ in the interval $0 < z < 1$. 
Hence, the probability of the root in $d$ is neither 0 nor 1. 

\textbf{Case 2:} Otherwise. Let the height be $h = 2k + 1$. The case of $k=0$ is easy.  
In the remainder of Case 2, assume $k \geq 1$. 

Let $f(x)=x^6 + 2x^5 - 2x^4 -6x^3 -3x^2 +2$. It is easy to see that, in the interval $0<x<1$,  the equation $f(x)=0$ has a unique solution. Let $x=\alpha$ be the solution. 
Then $\alpha < \forall x < 1 \quad f(x) < 0$.
We are going to show the following, by induction on $k \geq 1$. 

\begin{equation} \label{eq:alpha}
\alpha < \forall x < 1 \quad c_{\vee, 2k+1} (x) > 2^{k+1} 
\end{equation}

The case of $k=1$ is shown as follows. 
\begin{align}
c_{\wedge, 2} (x) &= (1 + x) (2 - x^2) \label{eq:andorh2cost}
\\
p_{\wedge, 2} (x) &= -x^2 ( x^2 - 2 ) \label{eq:andorh2prob}
\end{align}
\begin{align}
c_{\vee, 3} (x) - 2^2 &= c_{\wedge, 2} (x) ( 1+ p_{\wedge, 2} (x) ) - 2^2 \notag 
\\
&= (x - 1)f(x)
\end{align}

In the interval $\alpha < x < 1$, it holds that $c_{\wedge, 3} (x) - 2^2 > 0$. 
Thus, \eqref{eq:alpha} holds in the case of $k=1$. 

Next, we look at induction step of \eqref{eq:alpha}. 
The followings hold. 

\begin{align}
c_{\vee, 2k+3} (x) &= c_{\wedge, 2k+2} (x) ( 1 + p_{\wedge, 2k+2} (x) ) \notag
\\
&= c_{\wedge, 2} (p_{\wedge, 2k} (x) ) c_{\wedge, 2k} (x) 
( 1 + p_{\wedge, 2} ( p_{\wedge, 2k} (x) ) ) \notag
\\
&= c_{\wedge, 2k} (x) (1 + p_{\wedge, 2k} (x)) \times 
\frac
{c_{\wedge, 2} (p_{\wedge, 2k} (x) ) 
( 1 + p_{\wedge, 2} ( p_{\wedge, 2k} (x) ) )}
{1 + p_{\wedge, 2k} (x)} \notag
\\
&= c_{\vee, 2k+1} (x) \times 
\frac
{c_{\wedge, 2} (p_{\wedge, 2k} (x) ) 
( 1 + p_{\wedge, 2} ( p_{\wedge, 2k} (x) ) )}
{1 + p_{\wedge, 2k} (x)} 
\label{eq:original_odd1}
\end{align}

Here, in the interval $0 \leq t \leq 1$, the following is easily verified. 
\begin{equation}
\frac{c_{\wedge, 2} (t) ( 1 + p_{\wedge, 2} (t) )}{1 + t} 
= t^2 (1 - t^2)(3 - t^2) + 2
\label{eq:original_odd2}
\end{equation}

By \eqref{eq:original_odd1}, \eqref{eq:original_odd2} and by the induction hypothesis 
$c_{\vee, 2k+1} (x) > 2^{k+1} \,\,\, (\alpha < x < 1)$, it holds that 
$c_{\vee, 2k+3} (x) > 2^{k+2}  \,\,\, (\alpha < x < 1)$. 

Thus, we have shown \eqref{eq:alpha} for all positive integer $k$. 
Therefore, it holds that $c_{\wedge, 2k+1} (x) > 2^{k+1} $ in the interval $0 < x < 1 - \alpha$. 

Therefore, by Proposition~\ref{prop:4main}, the probability of the root in $d$ is neither 0 nor 1. 
\end{proof} 

\section{Comparison with Correlated Distributions} \label{section:comparison}

In \cite{LT07}, the main motive for Theorem~\ref{thm:lt4} is the following. 

 \begin{theorem} \label{thm:lt9}
 $($Liu and Tanaka. A corollary to Theorem 9 of $\cite{LT07})$ 
Suppose that $T$ is an AND-OR tree whose height is positive and even. 
The equilibrium among all IDs is strictly smaller than the equilibrium among all ditributions. 
To be more precise, the following holds.  

\begin{equation} \label{eq:lt9}
\max_{d :  \mathrm{ID}} \min_{A_D} C(A_D, d )
<
\max_d \min_{A_D} C(A_D, d ) 
\end{equation}

In the left-hand side, $d$ runs over all IDs. 
In the right-hand side, $d$ runs over all distributions. 
\end{theorem}

The above result extends to the case where the probability of the root is fixed. 
To be more precise, the following holds. 

\begin{theorem} \label{thm:idvscorelated}
Suppose that $T$ is an AND-OR tree whose height is greater than or equal to 2. 
Let $r$ be a real number such that $0 \leq r \leq 1$. 
Now, we restrict ourselves to the distribution such that the probability of the root is $r$. 
Then, the equilibrium among all IDs is strictly smaller than the equilibrium among all ditributions.  

\begin{equation} \label{eq:idvscorrelated}
\max_{d : \mathrm{ID}, r} \min_{A_D} C(A_D, d )
<
\max_{d : r} \min_{A_D} C(A_D, d ) 
\end{equation}

In the left-hand side, $d$ runs over all IDs such that the probability of the root is $r$. 
In the right-hand side, $d$ runs over all distributions such that the probability of the root is $r$.  
\end{theorem}

The proof of the theorem is given by using $i$-sets. 
Saks and Wigderson \cite{SW86} introduced the concept of a reluctant input. 
A truth assignment is said to be reluctant if it satisfies the following two requirements. 
(1) If an AND gate has the value 0, then exactly one child has the value 0; 
(2) If an OR gate has the value 1, then exactly one child has the value 1. 

Liu and Tanaka \cite{LT07} introduced the concepts of 0-set and 1-set. 
0-set (1-set, respectively) is the set of all reluctant truth assignments such that 
the root has the value 0 (1, respectively). 
In \cite{LT07}, it is shown that the right-hand side of \eqref{eq:lt9} is achieved only by the uniform distribution on the 1-set.

\begin{proof} (of Theorem~\ref{thm:idvscorelated}. Sketch) 
Suppose that $d_0$ is an ID with the following two properties: The probability of the root is $r$; And, $\min_{A_D} C(A_D, d_0)$ equals to the left-hand side of \eqref{eq:idvscorrelated}. 
And, let $A_0$ be a deterministic algorithm such that 
$C(A_0, d_0 ) = \min_{A_D} C(A_D, d_0 )$. 

For each $i \in \{ 0,1 \}$, let $d_\mathrm{unif.}^i$ be the uniform distribution on the $i$-set. 
And, let $d_\mathrm{mix}$ be the mixed strategy of $d_\mathrm{unif.}^0$ and $d_\mathrm{unif.}^1$ with weights $r : 1 - r$. 

\textbf{Case 1:} $r = 0$ or $1$. By Proposition~\ref{prop:4main}, it holds that $C(A_0, d_0 ) < C(A_0, d_\mathrm{mix})$. 

\textbf{Case 2:} Otherwise.  Then $0 < r < 1$. By Theorem~\ref{thm:main}, $d_0$ is the IID with root having probability $r$. Then, it is not hard to see that $C(A_0, d_0 ) < C(A_0, d_\mathrm{mix})$. 

Thus, we have shown the theorem. 
\end{proof}

For more precise on the cost of uniform distributions on the $i$-sets, see \cite[Lemma 3]{SN12}. 

\section{Conclusive Remarks} \label{section:conclusiveremarks}

Is Theorem~\ref{thm:lt4} is not hard to show? 
We close the current paper by putting a remark that a brutal induction does not work for the proof of Theorem~\ref{thm:lt4}. 

Suppose that $x=x_0$ maximizes $c_{\wedge, 2} (x)$. 
Then, it is not hard to see that $x=x_0$ does not maximize $c_{\vee, 3} (x)$. 
Thus, when we prove Theorem~\ref{thm:lt4}, we cannot carelessly mislead the reader by asserting the following: 
``Assume the statement of the theorem holds for the height = $h$. We investigate the case of the height = $h+1$. By the induction hypothesis, the restriction of a given ID to the left sub-tree is obviously an IID. And, the same thing holds for the right sub-tree. Hence, by the method of Lagrange multipliers, the statement of the theorem for height = $h+1$ is immediately shown."

\section*{Acknowledgment}

The authors would like to thank Kousuke Ogawa and Masahiro Kumabe for helpful discussions.


\end{document}